\documentclass[twoside,11pt]{article}

\usepackage{jmlr2e}

\usepackage{subcaption}
\usepackage{common}
\usepackage{macros}

\usepackage{lipsum}

\newcommand{\indep}{\mathbin{\rotatebox[origin=c]{90}{$\models$}}}

\renewcommand{\P}{\mathbb{P}}
\newcommand{\E}{\mathbb{E}}

\newcommand{\N}{\mathcal{N}}

\renewcommand{\H}{\mathcal{H}}
\newcommand{\D}{\mathcal{D}}
\renewcommand{\L}{\mathcal{L}}

\newcommand{\half}{\frac{1}{2}}

\newcommand{\w}{\mathbf{w}}

\newcommand{\x}{\mathbf{x}}
\newcommand{\n}{\mathbf{n}}
\newcommand{\y}{\mathbf{y}}
\newcommand{\hx}{\hat{\x}}

\newcommand{\z}{\mathbf{z}}

\newcommand{\xd}{x^{(i)}}
\newcommand{\yd}{y^{(i)}}

\newcommand{\be}{\begin{equation}}
\newcommand{\ee}{\end{equation}}
\newcommand{\bea}{\begin{eqnarray}}
\newcommand{\eea}{\end{eqnarray}}

\newcommand{\utext}[2]{\underbrace{#1}_{\text{#2}}}

\DeclareMathOperator{\TC}{TC}

\jmlrheading{18}{2018}{1-34}{01/17; Revised 4/18}{6/18}{17-646}{Alessandro Achille and Stefano Soatto} 

\ShortHeadings{Emergence of Invariance and Disentanglement}{Achille and Soatto}
\firstpageno{1}

\begin{document}

\title{Emergence of Invariance and Disentanglement \\ in Deep Representations\thanks{A preliminary version of this paper was presented at the ICML Workshop PADL on August 10, 2017.}} 

\author{\name Alessandro Achille \email achille@cs.ucla.edu \\
       \addr Department of Computer Science\\
       University of California\\
       Los Angeles, CA 90095, USA
       \AND
       \name Stefano Soatto \email soatto@cs.ucla.edu \\
       \addr Department of Computer Science\\
       University of California\\
       Los Angeles, CA 90095, USA}
\editor{Yoshua Bengio}

\maketitle

\begin{abstract}%
Using established principles from Statistics and Information Theory, we show that invariance to nuisance factors in a deep neural network is equivalent to information minimality of the learned representation, and that stacking layers and injecting noise during training naturally bias the network towards learning invariant representations. We then decompose the cross-entropy loss used during training and highlight the presence of an inherent overfitting term. We propose regularizing the loss by bounding such a term in two equivalent ways: One with a Kullbach-Leibler term, which relates to a PAC-Bayes perspective; the other using the information in the weights as a measure of complexity of a learned model, yielding a novel Information Bottleneck for the weights. Finally, we show that invariance and independence of the components of the representation learned by the network are bounded above and below by the information in the weights, and therefore are implicitly optimized during training. The theory enables us to quantify and predict sharp phase transitions between underfitting and overfitting of random labels when using our regularized loss, which we verify in experiments, and sheds light on the relation between the geometry of the loss function, invariance properties of the learned representation, and generalization error.
\end{abstract}



\section{Introduction}

Efforts to understand the empirical success of deep learning have followed two main lines: Representation learning and optimization. In optimization, a deep network is treated as a black-box family of functions for which we want to find parameters (\emph{weights}) that yield good generalization.
Aside from the difficulties due to the non-convexity of the loss function, the fact that  deep networks are heavily over-parametrized presents a theoretical challenge: The bias-variance trade-off suggests they may severely overfit; yet, even without explicit regularization, they perform remarkably well in practice. Recent work suggests that this is related to properties of the loss landscape and to the implicit regularization performed by stochastic gradient descent (SGD), 
but the overall picture is still hazy \citep{zhang2016understanding}. 

Representation learning, on the other hand, focuses on
the properties of the  representation
learned by the layers of the network (the \emph{activations})
while remaining largely agnostic to the particular optimization process used. In fact, the effectiveness of deep learning is often ascribed to the ability of deep networks to learn representations that are insensitive (invariant) to nuisances such as translations, rotations, occlusions, and also ``disentangled,'' that is, separating factors in the high-dimensional space of data \citep{bengio2009learning}.
Careful engineering of the architecture plays an important role in achieving insensitivity to simple geometric nuisance transformations, like translations and small deformations; however, more complex and dataset-specific nuisances still need to be learned. This poses a riddle:
\emph{If neither the architecture nor the loss function explicitly enforce invariance and disentangling, how can these properties emerge consistently in deep networks trained by simple generic optimization?}

\begin{figure}

\centering
\includegraphics[width=.5\linewidth]{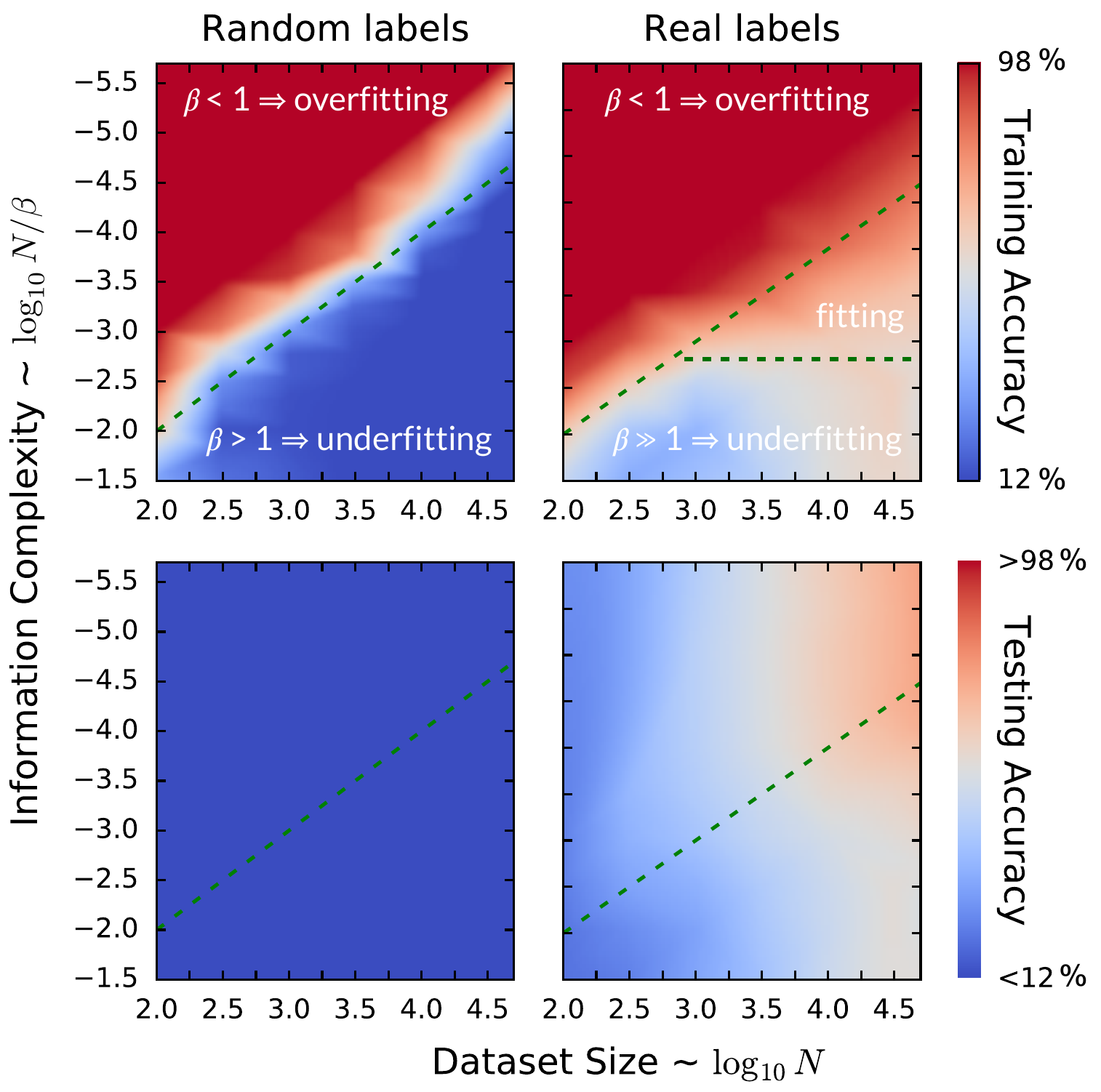}

\vspace*{-0.5em}

\caption{
\label{fig:phase-transiction}
\small
\textbf{(Left)} The AlexNet model of \cite{zhang2016understanding} achieves high accuracy (red) even when trained with random labels on CIFAR-10. Using the IB Lagrangian to limit information in the weights leads to a sharp transition to underfitting (blue) predicted by the theory (dashed line). To overfit, the network needs to memorize the dataset, and the  information needed grows linearly. \textbf{(Right)} For real labels, the information sufficient to fit the data without overfitting saturates to a value that depends on the dataset, but somewhat independent of the number of samples. Test accuracy shows a uniform blue plot for random labels, while for real labels it increases with the number of training samples, and is higher near the critical regularizer value $\beta=1$. }

\end{figure}

In this work, we address these questions by establishing information theoretic connections
between these concepts.
In particular, we show that: (a) a sufficient representation of the data is invariant if and only if it is minimal, \textit{i.e.}, it contains the smallest amount of information, although may not have small dimension; (b) the information in the representation, along with its total correlation (a measure of disentanglement) are tightly bounded by the information that the weights contain about the dataset;  (c) the information in the weights, which is related to overfitting \citep{hinton1993keeping}, flat minima \citep{hochreiter1997flat}, and a PAC-Bayes upper-bound on the test error (\Cref{sec:pac-bayes}), can be controlled by implicit or explicit regularization. Moreover, we show that adding noise during the training is a simple and natural way of biasing the network towards invariant representations. 

Finally, we perform several experiments with realistic architectures and datasets to validate the assumptions underlying our claims. In particular, we show that using the information in the weights to measure the complexity of a deep neural network (DNN), rather than the number of its parameters, leads to a sharp and theoretically predicted 
transition between overfitting and underfitting regimes for random labels, shedding light on the questions of \cite{zhang2016understanding}.

\subsection{Related work}

The Information Bottleneck (IB) was introduced by \citet{tishby2000information} as a generalization of minimal sufficient statistics that allows trading off fidelity (sufficiency) and complexity of a representation. In particular, the IB Lagrangian reduces finding a  minimal sufficient representation to a variational optimization problem. 
Later, \citet{tishby2015deep} and \citet{shwartz2017opening} advocated using the IB between the test data and the activations of a deep neural network, to study the sufficiency and minimality of the resulting representation.
In parallel developments, the IB Lagrangian was used as a regularized loss function for learning representation, leading to new information theoretic regularizers \citep{achille2016information,alemi2016deep,alemi2018fixing}.

In this paper, we introduce an IB Lagrangian between the weights of a network and the training data, as opposed to the traditional one between the activations and the test datum.  We show that the former can be seen both as a generalization of Variational Inference, related to \citet{hinton1993keeping}, and as a special case of the more general PAC-Bayes framework \citep{mcallester2013pac}, that can be used to compute high-probability upper-bounds on the test error of the network. One of our main contributions is then to show that, due to a particular duality induced by the architecture of deep networks, minimality of the weights (a function of the training dataset) and of the learned representation (a function of the test input) are connected: in particular we show that networks regularized either explicitly, or implicitly by SGD, are biased toward learning invariant and disentangled representations. The theory we develop could be used to explain the phenomena described in small-scale experiments in \citet{shwartz2017opening}, whereby the initial fast convergence of SGD is related to sufficiency of the representation, while the later asymptotic phase is related to compression of the activations: While SGD is seemingly agnostic to the property of the learned representation, we show that it does minimize the information in the weights, from which the compression of the activations follows as a corollary of our bounds. Practical implementation of this theory on real large scale problems is made possible by advances in Stochastic Gradient Variational Bayes \citep{kingma2013auto,kingma2015variational}.

Representations learned by deep networks are observed to be insensitive to complex nuisance transformations of the data. To a certain extent, this can be attributed to the architecture. For instance, the use of convolutional layers and max-pooling can be shown to yield insensitivity to local group transformations \citep{bruna2011classification,anselmi2016invariance,soatto2016visual}. But for more complex, dataset-specific, and in particular non-local, non-group transformations, such insensitivity must be acquired as part of the learning process, rather than being coded in the architecture. We show that a sufficient representation is maximally insensitive to nuisances if and only if it is minimal, allowing us to prove that a regularized network is naturally biased toward learning invariant representations of the data.

Efforts to develop a theoretical framework for representation learning include \citet{tishby2015deep} and \citet{shwartz2017opening},  who consider representations as stochastic functions that approximate minimal sufficient statistics, different from \citet{bruna2011classification} who construct representations as (deterministic) operators that are invertible in the limit, while exhibiting reduced sensitivity (``stability'') to small perturbations of the data. Some of the deterministic constructions are based on the assumption that the underlying data is spatially stationary, and therefore work best on textures and other visual data that are not subject to occlusions and scaling nuisances. \citet{anselmi2016invariance} develop a theory of invariance to locally compact groups, and aim to construct maximal (``distinctive'') invariants, like  \citet{sundaramoorthiPVS09} that, however, assume nuisances to be infinite-dimensional groups \citep{grenander93}. These efforts are limited by the assumption that nuisances have a group structure.
Such assumptions were relaxed by \cite{soatto2016visual} who advocate seeking for {\em sufficient} invariants, rather than {\em maximal} ones. We further advance this approach, but unlike prior work  on sufficient dimensionality reduction, we do not seek to minimize the dimension of the representation, but rather its information content, as prescribed by our theory. Recent advances in Deep Learning provide us with computationally viable methods to train high-dimensional models and predict and quantify observed phenomena such as convergence to flat minima and transitions from overfitting to underfitting random labels, thus bringing the theory to fruition.
Other theoretical efforts focus on complexity considerations, and explain the success of deep networks by ways of statistical or computational efficiency \citep{lee2017ability, bengio2009learning, lecun2012learning}.
``Disentanglement'' is an often-cited property of deep networks \citep{bengio2009learning}, but seldom formalized and studied analytically, although \citet{Ver2015maximally} has suggested studying it using the Total Correlation of the representation, also known as multi-variate mutual information, which we also use. 

We connect invariance properties of the representation to the geometry of the optimization residual, and to the phenomenon of \emph{flat minima} \citep{dinh2017sharp}.

Following \citep{mcallester2013pac}, we have also explored relations between our theory and the PAC-Bayes framework \citep{dziugaite2017computing}.
As we show, our theory can also be derived in the PAC-Bayes framework, without resorting to information quantities and the Information Bottleneck, thus providing both an independent and alternative derivation, and a theoretically rigorous  way to upper-bound the optimal loss function. The use of PAC-Bayes theory to study the generalization properties of deep networks has been championed by \cite{dziugaite2017computing}, who point out that minima that are flat in the sense of having a large volume, toward which stochastic gradient descent algorithms are implicitly or explicitly biased \citep{chaudhari2018cycles}, naturally relates to the PAC-Bayes loss for the choice of a normal prior and posterior on the weights. This has been leveraged by \cite{dziugaite2017computing} to compute non-vacuous PAC-Bayes error bounds, even for deep networks.

\section{Preliminaries}
\label{sec:notation}

A training set $\D=\set{\x,\y}$, where $\x = \set{\xd}_{i=1}^N$ and $\y=\set{\yd}_{i=1}^N)$, is a collection of $N$ randomly sampled data points $\xd_i$ and their associated (usually discrete) labels. The samples are assumed to come from an unknown, possibly complex, distribution $p_\theta(x,y)$, \hyphenation{pa-ram-e-tri-zed} parametrized by a parameter $\theta$. Following a Bayesian approach, we also consider $\theta$ to be a random variable, sampled from some unknown prior distribution $p(\theta)$, but this requirement is not necessary (see \Cref{sec:pac-bayes}).
A test datum $x$ is also a random variable. Given a test sample, our goal is to infer the random variable $y$, which is therefore referred to as our \emph{task}.

We will make frequent use of the following standard information theoretic quantities \citep{cover2012elements}:  Shannon entropy $H(x)=\E_p[-\log p(x)]$,  conditional entropy $H(x|y):=\E_{\bar{y}} [H(x|y=\bar{y})] = H(x,y) - H(y)$, (conditional) mutual information $I(x;y|z) = H(x|z) - H(x|y,z)$, Kullbach-Leibler (KL) divergence $KL(p(x) || q(x)) = \E_p[\log p/q]$,   cross-entropy $H_{p,q}(x) = \E_p[-\log q(x)]$, and total correlation $TC(z)$, which is also known as multi-variate mutual information and defined as 
\[ \textstyle \TC(z) = \KL{p(z)}{\prod_i p(z_i)},\]
where $p(z_i)$ are the marginal distributions of the components of $z$.
Recall that the KL divergence between two distributions is always non-negative and zero if and only if they are equal. In particular $\TC(z)$ is zero if and only if the components of $z$ are independent, in which case we say that $z$ is \emph{disentangled}.
We often use of the following identity:
\begin{equation*}
\label{eq:information-kl}
I(z;x) = \E_{x\sim p(x)} \KL{p(z|x)}{p(z)}.
\end{equation*}
We say that $x, z, y$ form a Markov chain, indicated with $x \rightarrow z \rightarrow y$, if $p(y|x,z) = p(y|z)$. The Data Processing Inequality (DPI) for a Markov chain $x \rightarrow z \rightarrow y$ ensures that $I(x;z) \ge I(x; y)$: If $z$ is a (deterministic or stochastic) function of $x$, it cannot contain more information about $y$ than $x$ itself (we cannot create new information by simply applying a function to the data we already have).

\subsection{General definitions and the Information Bottleneck Lagrangian}

We say that $z$ is a \emph{representation} of $x$ if $z$ is a stochastic function of $x$, or equivalently if the distribution of $z$ is fully described by the conditional $p(z|x)$. In particular we have the Markov chain $y \to x \to z$. We say that a representation $z$ of $x$ is \textbf{sufficient} for $y$ if $y \indep x \ | \ z$, or equivalently if $I(z;y)=I(x;y)$; it is \textbf{minimal} when $I(x;z)$ is smallest among sufficient representations. To study the trade-off between sufficiency and minimality,  \citet{tishby2000information}  introduces the Information Bottleneck Lagrangian
\begin{equation}
\L(p(z|x)) = H(y|z) + \beta\, I(z;x),
\label{eq:IBL}
\end{equation}
where $\beta$ trades off sufficiency (first term) and minimality (second term); in the limit $\beta \rightarrow 0$, the IB Lagrangian is minimized when $z$ is minimal and sufficient. It does not impose any restriction on disentanglement nor invariance, which we introduce next.

\subsection{Nuisances for a task}

A \textbf{nuisance} is any random variable that affects the observed data $x$, but is not informative to the task we are trying to solve. More formally, a random variable $n$ is a nuisance for the task $y$ if
$y \indep n$, or equivalently  $I(y;n)=0$.
Similarly, we say that the representation $z$ is \textbf{invariant} to the nuisance $n$ if
$z \indep n$, or $I(z;n)=0$. When $z$ is not strictly invariant but it minimizes $I(z;n)$ among all sufficient representations, we say that the representation $z$ is \textbf{maximally insensitive} to $n$.

One typical example of nuisance is a group $G$, such as translation
or rotation, acting on the data.
In this case, a deterministic representation $f$ is invariant to the nuisances
if and only if for all $g \in G$ we have $f(g\cdot x) = f(x)$.
Our definition however is more general in that it is not restricted to deterministic functions, nor to group nuisances.
An important consequence of this generality is that
the observed data $x$ can always be written as a deterministic
function of the task $y$ and of all nuisances $n$ affecting the data,
as explained by the following proposition.

\begin{prop}[Task-nuisance decomposition, Appendix \ref{lemma:task-nuisance-proof}]
\label{lemma:dist-factorization}
Given a joint distribution $p(x,y)$, where $y$ is a discrete random variable, we can always find a random variable $n$ independent of $y$ such that $x=f(y,n)$, for some deterministic function $f$.
\end{prop}

\section{Properties of optimal representations}

To simplify the inference process, instead of working directly with the observed high dimensional data $x$, we want to use a representation $z$
that captures and exposes only the information relevant for the task $y$.
Ideally, such a representation should be (a) \textbf{sufficient} for the task $y$, \emph{i.e.} $I(y;z) = I(y;x)$, so that information about $y$ is not lost; among all sufficient representations, it should be (b) \textbf{minimal}, \emph{i.e.} $I(z;x)$ is minimized, so that it retains as little about $x$ as possible, simplifying the role of the classifier; finally, it should be (c) \textbf{invariant} to the effect of nuisances $I(z;n) = 0$, so that the final classifier will not overfit to spurious correlations present in the training dataset between nuisances $n$ and labels $y$.
Such a representation, if it exists, would not be unique, since any bijective mapping preserves all these properties. We can use this to our advantage and further aim to make the representation (d) maximally  \textbf{disentangled}, \emph{i.e.,} choose the one(s) for which $\TC(z)$ is minimal. This simplifies the classifier rule,  since no information will be present in the higher-order correlations between the components of $z$.

Inferring a representation that satisfies all these properties may seem
daunting. However, in this section we show that we only need to enforce (a) sufficiency and (b) minimality, from which invariance and disentanglement follow naturally thanks to the stacking of noisy layers of computation in deep networks. We will then show that sufficiency and minimality of the learned representation can also be promoted easily through implicit or explicit regularization during the training process.

\begin{prop}[Invariance and  minimality, Appendix \ref{prop:invariance-minimality-proof}]
\label{prop:minimal-iff-invariant}
Let $n$ be a nuisance for the task $y$ and let $z$ be a sufficient representation of the input $x$. Suppose that 
$z$ depends on $n$ only through $x$ (\emph{i.e.,}  $n\to x\to z$). Then,
\[
I(z;n) \leq I(z;x) - I(x;y).
\]
Moreover, there is a nuisance $n$ such that
equality holds up to a (generally small) residual $\epsilon$
\[
I(z;n) = I(z;x) - I(x;y) - \epsilon,
\]
where  $\epsilon := I(z;y|n) - I(x;y)$. In particular
$0 \leq \epsilon \leq H(y|x) $, and $\epsilon=0$ whenever $y$ is a deterministic function of $x$.
 Under these conditions, a sufficient statistic $z$ is invariant (maximally insensitive) to nuisances if and only if it is minimal.
\end{prop}
\begin{rmk}
\textnormal{Since $\epsilon\leq H(y|x)$, and usually $H(y|x) = 0$ or at least $H(y|x) \ll I(x;z)$, we can generally ignore the extra term.}
\end{rmk}

An important consequence of this proposition is that we can construct invariants by simply reducing the amount of information $z$ contains about $x$,
while retaining the minimum amount $I(z;x)$ that we need for the task $y$. This provides the network a way to automatically learn invariance to complex nuisances, which is complementary to the invariance imposed by the architecture.
Specifically, one way of enforcing minimality explicitly, and hence invariance, is through the IB Lagrangian.
\begin{cor}[Invariants from the Information Bottleneck]
\label{cor:ib-invariant}
Minimizing the IB Lagrangian
\[
\L(p(z|x)) = H(y|z) + \beta\, I(z;x),
\]
in the limit $\beta\to 0$, yields a sufficient invariant representation $z$ of the test datum $x$ for the task $y$.
\end{cor}
Remarkably, the IB Lagrangian can be seen as the standard cross-entropy loss, plus a regularizer $I(z;x)$ that promotes invariance. This fact, without proof, is implicitly used in \cite{achille2016information}, who also provide an efficient algorithm 
to perform the optimization. \cite{alemi2016deep} also propose a related algorithm and empirically show improved resistance to adversarial nuisances.
In addition to modifying the cost function, invariance can also be fostered by choice of architecture:
\begin{cor}[Bottlenecks promote invariance]
\label{cor:noise}
Suppose we have the Markov chain of layers
\[
x \to z_1 \to z_2 ,
\]
and suppose that there is a communication or computation bottleneck between $z_1$ and $z_2$
such that $I(z_1;z_2) < I(z_1;x)$. Then, if $z_2$ is still sufficient,
it is more invariant to nuisances than $z_1$.
More precisely, for all nuisances $n$ we have
$I(z_2;n) \leq I(z_1;z_2) - I(x;y)$.
\end{cor}
Such a bottleneck can happen
for example because $\dim(z_2) < \dim(z_1)$, \emph{e.g.}, after a pooling layer, or because the channel
between $z_1$ and $z_2$ is noisy, \emph{e.g.}, because of dropout.

\begin{prop}[Stacking increases invariance]
Assume that we have the Markov chain of layers
\[
x \to z_1 \to z_2 \to \ldots \to z_L,
\]
and that the last layer $z_L$ is sufficient of $x$ for $y$.
Then $z_L$ is more insensitive to nuisances than all the preceding layers.
\end{prop}
Notice, however, that the above corollary does not simply imply that the more layers the merrier, as it assumes that one has successfully trained the network ($z_L$ is sufficient), which becomes increasingly difficult as the size grows. Also note that in some architectures, such as ResNets \citep{he2016deep}, the layers do not necessarily form a Markov chain because of skip connections; however, their ``blocks'' still do.

\begin{prop}[Actionable Information]
When $z = f(x)$ is a deterministic invariant, if it minimizes the IB Lagrangian it also maximizes Actionable Information \citep{soatto2013actionable}, which is ${\cal H}(x) := H(f(x))$. 
\end{prop}
Although \citet{soatto2013actionable} addressed maximal invariants, we only consider sufficient invariants, as advocated by \citep{soatto2016visual}.

\subsection*{Information in the weights}

 Thus far we have discussed properties of representations in generality, regardless of how they are implemented or learned.
Given a source of data (for example randomly generated, or from a fixed dataset), and given a (stochastic) training algorithm, the output weight $w$ of the training process can be thought as a random variable (that depends on the stochasticity of the initialization, training steps and of the data). We can therefore talk about the information that the weights contain about the dataset $\D$ and the training procedure, which we denote by $I(w; {\cal D})$.

Two extreme cases consist of the trivial settings where we use the weights to memorize the dataset (the most extreme form of overfitting), or where the weights are constant, or pure noise (sampled from a process that is independent of the data). In between, the amount of information the weights contain about the training turns out to be an important quantity both in training deep networks, as well as in establishing properties of the resulting representation, as we discuss in the next section. 

Note that in general we do not need to compute and optimize the quantity of information in the weights. Instead, we show that we can \emph{control} it, for instance by injecting noise in the weights, drawn from a chosen distribution, in an amount that can be modulated between zero (thus in theory allowing full information about the training set to be stored in the weights) to an amount large enough that no information is left. We will leverage this property in the next sections to perform regularization.

\section{Learning minimal weights}
\label{sec:minimal-weights}
{
In this section, we let $p_\theta(x,y)$ be an (unknown) distribution from which we randomly sample a dataset $\D$. The parameter $\theta$ of the distribution is also assumed to be a random variable with an (unknown) prior distribution $p(\theta)$. For example $p_\theta$ can be a fairly general generative model for natural images, and $\theta$ can be the parameters of the model that generated our dataset.
We then consider a deep neural network that implements a map $x \mapsto f_w(x) := q(\,{\cdot}\, |x, w)$ from an input $x$ to a class distribution $q(y|x,w)$.%
\footnote{
We use $p$ to denote the real (and unknown) data distribution, while $q$ denotes approximate distributions that are optimized during training.
}
In full generality, and following a Bayesian approach, we let the weights $w$ of the network be sampled from a parametrized distribution $q(w|\D)$,%
whose parameters are optimized during training.%
\footnote{Note that, while the two are somewhat related, here by $q(w|\D)$ we denote the output distribution of the weights after training with our choice algorithm on the dataset $\D$, and not the Bayesian posterior of the weights given the dataset, which would be denoted $p(w|\D)$. When $q(w|\D)$ is a Dirac delta at a point, we recover the standard loss function for a MAP estimate of the weights.} The network is then trained in order to minimize the expected cross-entropy loss}%
\footnote{
Note that for generality here we treat the dataset $\D$ as a random variable. In practice, when a single dataset is given, the expectation w.r.t. the dataset can be ignored. 
}
\[
H_{p,q}(\y | \x, w) = \E_{\D=(\x,\y)} {\E_{w\sim q(w|\D)} } \sum_{i=1}^N -\log q(\yd|\xd,w),
\]
in order for $q(y | x, w)$ to approximate $p_\theta(y|x)$.

One of the main problems in optimizing a DNN is that
the cross-entropy loss in notoriously prone to overfitting. In fact, one can easily minimize it even for completely random labels (see \cite{zhang2016understanding}, and \Cref{fig:phase-transiction}). {The fact that, somehow, such highly over-parametrized functions manage to generalize when trained on real labels has puzzled theoreticians and prompted some to wonder whether this may be inconsistent with the intuitive interpretation of the bias-variance trade-off theorem, whereby unregularized complex models should overfit wildly. However, as we show next, there is no inconsistency if one measures complexity by the information content, and not the dimensionality, of the weights.}

To gain some insights about the possible causes of over-fitting, we can
use the following decomposition of the cross-entropy loss
(we refer to \Cref{sec:proofs} for the proof and the precise definition of each term):
\begin{equation}
\label{eq:decomposition}
H_{p,q}(\y|\x,w)
= \utext{H(\y|\x,\theta)}{intrinsic error} + \utext{I(\theta;\y|\x, w)}{sufficiency} + 
\E_{\x,w}\utext{\KL{p(\y|\x,w)}{q(\y|\x,w)}}{efficiency} - \utext{I(\y;w|\x, \theta)}{overfitting}.
\end{equation}
The first term of the right-hand side of \eqref{eq:decomposition} relates to the intrinsic error that we would commit in predicting the labels even if we knew the underlying data distribution $p_\theta$; the second term measures how much information that the dataset has about the parameter $\theta$ is captured by the weights, the third term relates to the efficiency of the model and the class of functions $f_w$ with respect to which the loss is optimized.
The last, and only negative, term relates to how much information about the labels, but uninformative of the underlying data distribution, is memorized in the weights. Unfortunately, without implicit or explicit
regularization, the network can minimize the cross-entropy loss (LHS), by just maximizing the last term of \cref{eq:decomposition}, {\em i.e.,} by memorizing the dataset, which yields poor generalization.

To prevent the network from doing this, we can neutralize the effect of the negative term by adding it back to the loss function, leading to a regularized loss $L=H_{p,q}(\y | \x, w) + I(\y;w|\x,\theta)$.
However, computing, or even approximating, the value of $I(\y,w|\x,\theta)$ is at least as difficult as fitting the model itself.

We can, however, add an upper bound to $I(\y;w|\x,\theta)$ to obtain the desired result. In particular, we explore two alternate paths that lead to equivalent conclusions under different premises and assumptions: In one case, we use a PAC-Bayes upper-bound, which is $\KL{q(w|\D)}{p(w)}$ where $p(w)$ is an arbitrary prior. In the other, we use the IB Lagrangian and upper-bound it with the information in the weights $I(w;\D)$. We discuss this latter approach now, and look at the PAC-Bayes approach in \Cref{sec:pac-bayes}.

Notice that to successfully learn the distribution $p_\theta$, we only need to memorize in $w$ the information about the latent parameters $\theta$, that is we need $I(\D;w) = I(\D;\theta) \leq H(\theta)$, which is bounded above by a constant. On the other hand, to overfit, the term
$I(\y;w|\x,\theta) \leq I(\D;w|\theta)$ needs to grow linearly with the number of training samples $N$. We can exploit this fact to prevent overfitting by adding a Lagrange multiplier $\beta$ to make the amount of information a constant with respect to $N$, leading to the regularized loss function
\begin{equation}
\label{eq:variational-regularizer}
\L(q(w|\D)) = H_{p,q}(\y|\x,w) + \beta I(w;\D),
\end{equation}
which, remarkably, has the same general form of an IB Lagrangian, and in particular is similar to \eqref{eq:IBL}, but now interpreted as a function of the weights $w$ rather than the activations $z$. This use of the IB Lagrangian is, to the best of our knowledge, novel, as the role of the Information Bottleneck has thus far been confined to characterizing the activations of the network, and not as a learning criterion. \Cref{eq:variational-regularizer} can be seen as a generalization of other suggestions in the literature:

\paragraph{IB Lagrangian, Variational Learning and Dropout.}
Minimizing the information stored at the weights $I(w; \D)$ was proposed as far back as \citet{hinton1993keeping} as a way of simplifying neural networks, but no efficient algorithm to perform the optimization was known at the time. For the particular choice $\beta=1$, the IB Lagrangian reduces to the variational lower-bound (VLBO) of the marginal log-likelihood $p(\y|\x)$. Therefore, minimizing \cref{eq:variational-regularizer} can also be seen as a generalization of variational learning.
A particular case of this was studied by \citet{kingma2015variational},
who first showed that a generalization of Dropout, called Variational Dropout, could be used in conjunction with the \emph{reparametrization trick} \cite{kingma2013auto}
 to minimize the loss efficiently.

\paragraph{Information in the weights as a measure of complexity.}
Just as \citet{hinton1993keeping} suggested, we also advocate using the information regularizer 
$I(w;\D)$ as a measure of the effective complexity of a network, rather than the number of parameters $\dim(w)$,
which is merely an upper bound on the complexity.
As we show in experiments, this allows us to recover a version of the bias-variance trade-off where networks with lower information complexity underfit the data, and networks with higher complexity overfit. In contrast, there is no clear relationship between number of parameters and overfitting \citep{zhang2016understanding}.
Moreover, for random labels the information complexity allows us to precisely predict the overfitting and underfitting behavior of the network (\Cref{sec:expm}).

\subsection{Computable upper-bound to the loss}

Unfortunately, computing $I(w,\D) = \E_\D \KL{q(w|\D)}{q(w)}$ is still too complicated, since it requires us to know the marginal $q(w)$ over all possible datasets and trainings of the network. To avoid computing this term, we can use the more general upper-bound
\begin{align*}
\E_\D \KL{q(w|\D)}{q(w)} &\leq \E_\D \KL{q(w|\D)}{q(w)} + \KL{q(w)}{p(w)}\\
&= \E_\D \KL{q(w|\D)}{p(w)},
\end{align*}
where $p(w)$ is any fixed distribution of the weights. Once we instantiate the training set, we have a single sample of ${\cal D}$, so the expectation over $\D$ becomes trivial. This gives us the following upper bound to the optimal loss function
\begin{equation}
\label{eq:completely-general-loss}
\L(q(w|\D)) = H_{p,q}(\y|\x,w) + \beta \KL{q(w|\D)}{p(w)}
\end{equation}
Generally, we want to pick $p(w)$ in order to give the sharpest upper-bound, and to be a fully factorized distribution, \textit{i.e.}, a distribution with independent components, in order to make the computation of the KL term easier. The sharpest upper-bound to $\KL{q(w|\D)}{q(w)}$ that can be obtained using a factorized distribution $p$ is obtained when $p(w) := \tilde{q}(w) = \prod_i q(w_i)$ where $q(w_i)$ denotes the marginal distributions of the components of $q(w)$. Notice that. once a training procedure is fixed, this may be approximated by training multiple times and approximating each marginal weight distribution. With this choice of prior, our final loss function becomes
\begin{equation}
\label{eq:general-loss}
\L(q(w|\D)) = H_{p,q}(\y|\x,w) + \beta \KL{q(w|\D)}{\tilde{q}(w)}
\end{equation}
for some fixed distribution $\tilde{q}$ that approximates the real marginal distribution $q(w)$.
The IB Lagrangian for the weights in \cref{eq:variational-regularizer} can be seen as a generally intractable special case of \cref{eq:general-loss} that gives the sharpest upper-bound to our desired loss in this family of losses.

In the following, to keep the notation uncluttered, we will denote our upper bound $\KL{q(w|\D)}{\tilde{q}(w)}$ to the mutual information $I(w;\D)$ simply by $\tilde{I}(w;\D)$, where
\[
\tilde{I}(w;\D) := \KL{q(w|\D)}{\tilde{q}(w)} = \KL{q(w|\D)}{\textstyle \prod_i q(w_i)}.
\]

\subsection{Bounding the information in the weights of a network}
\label{sec:information-bound}

To derive precise and empirically verifiable statements about $\tilde{I}(w;\D)$, we need a setting where this can be expressed analytically and optimized efficiently on standard architectures. 
To this end, following \cite{kingma2015variational}, we make the following modeling choices.

\paragraph{Modeling assumptions.} Let $w$ denote the vector containing all the parameters (weights) in the network,
and let $W^k$ denote the weight matrix at layer $k$.
We assume an improper log-uniform prior on $w$,
that is $\tilde{q}(w_i) = c/|w_i|$. Notice that this is the only scale-invariant prior \citep{kingma2015variational}, and closely matches the real marginal distributions of the weights in a trained network \citep{achille2016information};  we parametrize the weight distribution $q(w_i|\D)$ during training as
\[
w_i|\D \sim \epsilon_i \hat{w}_i,
\]
where $\hat{w}_i$ is a learned mean, and  $\epsilon_i \sim \log \N (-\alpha_i/2, \alpha_i)$ is i.i.d. multiplicative log-normal noise with mean 1 and variance $\exp(\alpha_i) - 1$.%
\footnote{For a log-normal  $\log \N(\mu, \sigma^2)$ mean and variance are respectively 
$\exp(\mu + \sigma^2/2)$ and $[\exp(\sigma^2)-1]\exp(2\mu + \sigma^2)$.
}
Note that while \cite{kingma2015variational} uses this parametrization as a local approximation of the Bayesian posterior for a given (log-uniform) prior, we rather \emph{define} the distribution of the weights $w$ after training on the dataset $\D$ to be $q(w|\D)$.

\begin{prop}[Information in the weights, \Cref{prop:information-weight-proof}]
Under the previous modeling assumptions, the upper-bound to the information that the weights contain about the dataset is 
\[
{I(w;\D)} \leq \tilde{I}(w;\D) = - \half \sum_{i=1}^{\dim(w)} \log \alpha_i + C,
\]
where the constant $C$ is arbitrary due to the improper prior.
\end{prop}

\begin{rmk}[On the constant $C$]
\label{rmk:constant-C}
To simplify the exposition, since the optimization is unaffected by any additive constant, in the following we abuse the notation and, under the modeling assumptions stated above, we rather \emph{define}
$\tilde{I}(w;\D) := - \half \sum_{i=1}^{\dim(w)} \log \alpha_i $. \citet{neklyudov2017structured} also suggest a principled way of dealing with the arbitrary constant by using a proper log-uniform prior.
\end{rmk}
Note that computing and optimizing this upper-bound to the information in the weights is relatively simple and efficient using the reparametrization trick of \cite{kingma2015variational}.

\subsection{Flat minima have low information}

Thus far we have suggested that adding the explicit information regularizer $I(w; \D)$ prevents the network from memorizing the dataset and thus avoid overfitting, which we also confirm empirically in \Cref{sec:expm}.
However, real networks are not commonly trained with this regularizer, thus seemingly undermining the theory.
However, even when not explicitly present, the term $I(w;\D)$ is implicit in the use of SGD. In particular, \citet{chaudhari2018cycles} show that, under certain conditions, SGD introduces an entropic bias of  a very similar form to the information in the weights described thus far, where the amount of information can be controlled by the learning rate and the size of mini-batches.

Additional indirect empirical evidence is provided by the fact that some variants of SGD \citep{chaudhari2016entropy} bias the optimization toward ``flat minima'', that are local minima whose Hessian has mostly small eigenvalues.
These minima can be interpreted exactly as having low information $I(w;\D)$, as  suggested early on by \citet{hochreiter1997flat}: {Intuitively, since the loss landscape is locally flat, the weights may be stored at lower precision without incurring in excessive inference error.
As a consequence of previous claims, we can then see flat minima as having better generalization properties and, as we will see in \Cref{sec:duality}, the associated representation of the data is more insensitive to nuisances and more disentangled.}
For completeness, here we derive a more precise relationship between flatness (measured by the nuclear norm of the loss Hessian), and the information content based on our model.
\begin{prop}[Flat minima have low information, Appendix \ref{prop:flat-minima-proof}]
\label{prop:flat-minima}
Let $\hat w$ be a local minimum of the cross-entropy loss $H_{p,q}(\y|\x,w)$, and let $\H$ be the Hessian at that point.
Then, for the optimal choice of the
posterior $w|\D = \epsilon \odot \hat w$ centered at $\hat w$ that optimizes the IB Lagrangian, we have
\[
{I(w;\D)} \leq\tilde{I}(w;\D) \leq \half K [\log \norm{\hat{w}}_2^2 + \log \norm{\H}_* - K\log (K^2 \beta/2)]
\]
where $K = \dim(w)$ and $\| \cdot \|_*$ denotes the nuclear norm.
\end{prop}

Notice that a converse inequality, that is, low information implies flatness,
needs not hold, so there is no contradiction with the results of \citet{dinh2017sharp}.
Also note that for $\tilde{I}(w;\D)$ to be invariant to reparametrization one has to consider the constant $C$, which we have ignored (Remark \ref{rmk:constant-C}). The connection between flatness and overfitting has also been studied by \cite{neyshabur2017exploring}, including the effect of the number of parameters in the model.

In the next section, we prove one of our main results, that networks with low information in the weights realize invariant and disentangled representations. Therefore, invariance and disentanglement emerge naturally when training a network with implicit (SGD) or explicit (IB Lagrangian) regularization, and are related to flat minima.

\section{Duality of the Bottleneck}
\label{sec:duality}

The following proposition gives the fundamental link in our model between information in the weights, and hence flatness of the local minima, minimality of the representation, and disentanglement.
\begin{prop}[Appendix \ref{prop:information-exact-computation}]
Let $z=Wx$, and assume as  before $W = \epsilon \odot \hat{W}$, with 
$\epsilon_{i,j} \sim  \log \N(-\alpha_i/2, \alpha_i)$.
Further assume that the marginals of $p(z)$ and $p(z|x)$ are both approximately Gaussian
(which is reasonable for large $\dim(x)$ by the Central Limit Theorem). Then,
\begin{equation}
I(z;x) + \TC(z) =
 - \half \sum_{i=1}^{\dim(z)} \E_x  \log \frac{\tilde{\alpha}_i \hat{W}_i^2 \cdot  x^2}{\hat{W}_i \cdot \Cov(x) \hat{W}_i  + \tilde{\alpha}_i \hat{W}_i^2 \cdot \E (x^2)},
\end{equation}
where $W_i$ denotes the $i$-th row of the matrix $W$, and $\tilde{\alpha_i}$ is the noise variance $\tilde{\alpha}_i = \exp(\alpha_i) - 1$. In particular, $I(z;x) + \TC(z)$ is a monotone decreasing function of the weight variances $\alpha_i$.
\end{prop}

The above identity is difficult to apply in practice,
but with some additional hypotheses, we can derive a cleaner uniform tight bound on $I(z;x) + \TC(z)$.

\begin{prop}[Uniform bound for one layer, Appendix \ref{prop:uniform-bound-proof}]
\label{cor:single-layer-bound}
Let $z=Wx$, where $W = \epsilon \odot \hat{W}$, where $\epsilon_{i,j} \sim \log \N(-\alpha/2,\alpha)$; assume that the components of $x$ are uncorrelated, and that their kurtosis is uniformly bounded.%
\footnote{
    This is a technical hypothesis, always satisfied if the components $x_i$ are IID, {(sub-)}Gaussian, or with uniformly bounded support.
}
Then, there is a strictly increasing function $g(\alpha)$ s.t. we have the uniform bound
\[
g(\alpha) \leq \frac{I(x;z) + TC(z)}{\dim(z)} \leq g(\alpha) + c,
\]
where $c=O(1/\dim(x))\leq 1$,
$g(\alpha)= - \log\,(1-e^{-\alpha})/2$
and $\alpha$ is related to $\tilde{I}(w;\D)$ by  $\alpha = \exp\set{-I(W;\D)/\dim(W)}$.
In particular, $I(x;z) + TC(z)$ is tightly bounded by $\tilde{I}(W;\D)$ and increases strictly with it.
\end{prop}
The above theorems tells us that whenever we decrease the information in the weights, either by explicit regularization, or by implicit regularization (\textit{e.g.}, using SGD), we automatically
improve the minimality, and hence, by Proposition \ref{prop:minimal-iff-invariant}, the invariance, and the disentanglement of the learner representation. In particular, we obtain as a corollary that SGD is biased toward learning invariant and disentangled representations of the data.
Using the Markov property of the layers, we can easily extend this bound to multiple layers:
\begin{cor}[Multi-layer case, Appendix \ref{cor:multi-layer-proof}]
\label{cor:multi-layer-bound}
Let $W^k$ for $k=1,...,L$ be weight matrices, with $W^k = \epsilon^k \odot \hat{W}^k$ and $\epsilon^k_{i,j}=\log \N(-\alpha^k/2,\alpha^k)$, and let $z_{i+1} = \phi(W^k z_k)$,
where $z_0 = x$ and $\phi$ is any nonlinearity. Then,
\[
I(z_L;x) \leq \min_{k<L} \set{\dim(z_k) \big[g(\alpha^k) + 1 \big] }
\]
where $\alpha^k = \exp\set{-I(W^k;\D)/\dim(W^k)}$.
\end{cor}

\begin{rmk}[Tightness]
\textnormal{
While the bound in Proposition \ref{cor:single-layer-bound} is tight, the bound in the multilayer case
needs not be. This is to be expected: Reducing the information in the weights creates a bottleneck,
but we do not know how much information about $x$ will actually go through this bottleneck.
Often, the final layers will let most of the information through, while initial layers will drop the most.
}
\end{rmk}

\begin{rmk}[Training-test transfer]
\textnormal{
We note that we did not make any (explicit) assumption about the test set having the same distribution of the training set. Instead, we make the less restrictive assumption of sufficiency: If the test distribution is entirely different from the training one -- one may not be able to achieve sufficiency. This prompts interesting questions about measuring the distance between {\em tasks} (as opposed to just distance between distributions), which will be studied in future work.
}
\end{rmk}

\section{Connection with PAC-Bayes bounds}
\label{sec:pac-bayes}

In this section we show that using a PAC-Bayes bound, we arrive at the same regularized loss function \cref{eq:general-loss} we obtained using the Information Bottleneck, without the need of any approximation. By Theorem 2 of \citet{mcallester2013pac}, we have that for any fixed $\lambda>1/2$,  prior $p(w)$, and any weight distribution $q(w|\D)$, the test error $L^{\text{test}}(q(w|\D))$ that the network commits using the weight distribution $q(w|\D)$ is upper-bounded in expectation by
\begin{equation}
\label{eq:pac-bayes-bound}
\E_\D [L^{\text{test}}(q(w|\D))] \leq \frac{1}{N(1-\frac{1}{2\lambda})} \Big(H_{p,q}(\y|\x,w) + \lambda L_{\text{max}} \E_\D[\KL{q(w|\D)}{p(w)}] \Big),
\end{equation}
where $L_{\text{\emph{max}}}$ is the maximum per-sample loss function, which for a classification problem we can assume to be upper-bounded, for example by clipping the cross-entropy loss at chance level.
Notice that right hand side coincides, modulo a multiplicative constant, with \cref{eq:completely-general-loss} that we derived as an approximation of the IB Lagrangian for the weights (\cref{eq:variational-regularizer}).

Now, recall that since we have
\begin{align*}
\E_\D [\KL{q(w|\D)}{q(w)}] &= \E_\D[\KL{q(w|\D)}{p(w)}] - \KL{q(w)}{p(w)}\\
&\leq \E_\D[\KL{q(w|\D)}{p(w)}],
\end{align*}
the sharpest PAC-Bayes upper-bound to the test error is obtained when $p(w) = q(w)$, in which case \cref{eq:pac-bayes-bound} reduces (modulo a multiplicative constant) to the IB Lagrangian of the weights. That is, the IB Lagrangian for the weights can be considered as a special case of PAC-Bayes giving the sharpest bound.

Unfortunately, as we noticed in \Cref{sec:minimal-weights}, the joint marginal $q(w)$ of the weights is not tractable. To circumvent the problem, we can instead consider that the sharpest PAC-Bayes upper-bound that can be obtained using a tractable factorized prior $p(w)$, which is obtained exactly when $p(w) = \tilde{q}(w) = \prod_i q(w_i)$ is the product of the marginals, leading again to our practical loss \cref{eq:general-loss}.

On a last note, recall that under our modeling assumptions the marginal $\tilde{q}(w)$ is assumed  to be an improper log-uniform distribution. While this has the advantage of being a non-informative prior that closely matches the real marginal of the weights of the network, it also has the disadvantage that it is only defined modulo an additive constant, therefore making the bound on the test error vacuous under our model.

The PAC-Bayes bounds has also been used by \citet{dziugaite2017computing} to study the generalization property of deep neural networks and their connection with the optimization algorithm. They use a Gaussian prior and posterior, leading to a non-vacuous generalization bound.

\section{Empirical validation}
\label{sec:expm}

\subsection{Transition from overfitting to underfitting}
As pointed out by  \cite{zhang2016understanding}, when a standard convolutional neural network (CNN) is trained on CIFAR-10 to fit random labels, the network is able to (over)fit them perfectly. This is easily explained in our framework: It means that the network is complex enough to memorize all the labels but, as we show here, it has to pay a steep price in terms of information complexity of the weights (\Cref{fig:IR2}) in order to do so. {On the other hand, when the information in the weights is bounded using and information regularizer, overfitting is prevented in a theoretically predictable way. 

In particular, in the case of completely random labels, we have $I(\y;w|\x,\theta) = I(\y;w) \leq I(w;\D)$, {where the first equality holds since $\y$ is by construction random, and therefore independent of $\x$ and $\theta$. In this case, the inequality used to derive \cref{eq:variational-regularizer} is an equality, and the IBL is an optimal regularizer, and, regardless of the dataset size $N$, for $\beta>1$ it should completely prevent memorization, while for $\beta<1$ overfitting is possible. To see this, notice that since the labels are random, to decrease the classification error by $\log |\mathcal{Y}|$, where $|\mathcal{Y}|$ is the number of possible classes, we need to memorize a new label. But to do so, we need to store more information in the weights of the network, therefore increasing the second term $I(w;\D)$ by a corresponding quantity. This trade-off is always favorable when $\beta<1$, but it is not when $\beta>1$. Therefore, the theoretically the optimal solution to $\cref{eq:IBL}$ is to memorize all the labels in the first case, and not memorize anything in the latter. 

As discussed, for real neural networks we cannot directly minimize $\cref{eq:IBL}$, and we need to use a computable upper bound to $I(w;\D)$ instead (\Cref{sec:information-bound}). Even so, the empirical behavior of the network, shown in \Cref{fig:phase-transiction}, closely follows this prediction, and for various sizes of the dataset clearly shows a phase transition between overfitting and underfitting near the critical value $\beta=1$. }
Notice instead that for real labels the situation is different: The model is still able to overfit when $\beta<1$, but importantly there is a large interval of $\beta>1$ where the model can fit the data \emph{without} overfitting to it. Indeed, as soon as $\beta N \propto I(w;\D)$ is larger than the constant $H(\theta)$, the model trained on real data fits real labels without excessive overfitting (\Cref{fig:phase-transiction}). 

{Notice that, based on this reasoning, we expect the presence of a phase transition between an overfitting and an underfitting regime at the critical value $\beta=1$ to be largely independent on the network architecture: To verify this, we train different architectures on a subset of 10000 samples from CIFAR-10 with random labels. As we can see on the left plot of \Cref{fig:IR2}, even very different architectures show a phase transition at a similar value of $\beta$. We also notice that in the experiment ResNets has a sharp transition close to the critical $\beta$. }

In the right plot of \Cref{fig:IR2} we measure the quantity information in the weights for different levels of corruption of the labels. To do this, we fix $\beta<1$ so that the network is able to overfit, and for various level of corruption we train until convergence, and then compute $I(w;\D)$ for the trained model.
As expected, increasing the randomness of the labels increases the quantity of information we need to fit the dataset. For completely random labels, $I(w;\D)$ increases by $\sim 3$ nats/sample, which the same order of magnitude as the quantity required to memorize a 10-class labels ($2.30$ nats/sample), as shown in \Cref{fig:IR2}.

\begin{figure}
\centering
\begin{subfigure}{.45\linewidth}
\centering
\includegraphics[height=4.0cm]{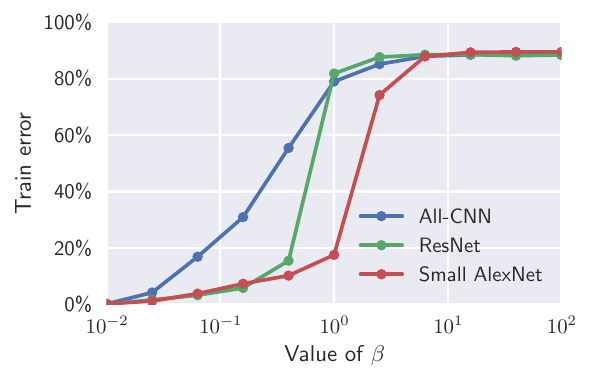}
\end{subfigure}
\begin{subfigure}{.45\linewidth}
\centering
\includegraphics[height=4.1cm,width=6cm]{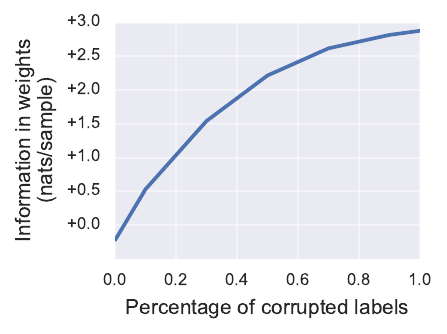}
\end{subfigure}
\vspace*{-.5em}
\caption{
\label{fig:IR2}
{\textbf{(Left)} Plot of the training error on CIFAR-10 with random labels as a function of the parameter $\beta$ for different models (see the appendix for details). As expected, all models show a sharp phase transition from complete overfitting to underfitting before the critical value $\beta=1$.  }
\textbf{(Right)} We measure the quantity of information in the weights necessary to overfit as we vary the percentage of corrupted labels under the same settings of \Cref{fig:phase-transiction}. To fit increasingly random labels, the network needs to memorize more information in the weights; the increase needed to fit entirely random labels is about the same magnitude as the size of a label (2.30 nats/sample).}
\end{figure}

\subsection{Bias-variance trade-off}

The Bias-Variance trade-off is sometimes informally stated as saying that low-complexity models tend to underfit the data, while excessively complex models may instead overfit, so that one should select an adequate intermediate complexity. This is apparently at odds with the common practice in Deep Learning, where increasing the depth or the number of weights of the network, and hence increasing the ``complexity'' of the model measured by the number of parameters, does not seem to induce overfitting. Consequently, a number of alternative measures of complexity have been proposed that capture the intuitive bias-variance trade-off curve, such as different norms of the weights \citep{neyshabur2015path}.

From the discussion above, we have seen that the quantity of information in the weights, or alternatively its computable upperbound $\tilde{I}(w;\D)$, also provides a natural choice to measure model complexity in relation to overfitting. In particular, we have already seen that models need to store increasingly more information to fit increasingly random labels (\Cref{fig:IR2}). In \Cref{fig:bias-variance} we show that by controlling $\tilde{I}(w;\D)$, which can be done easily by modulating $\beta$, we recover the right trend for the bias-variance tradeoff, whereas models with too little information tend to underfit, while models memorizing too much information tend to overfit.

\begin{figure}
\centering
\includegraphics[height=4.2cm]{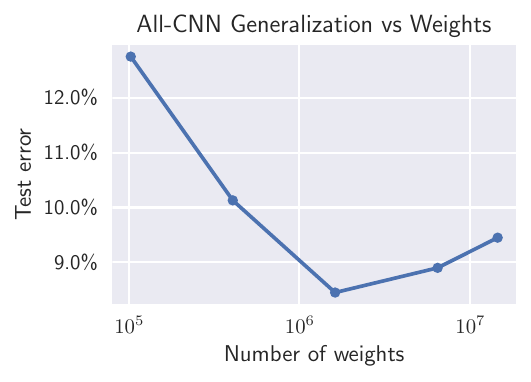}
\hspace{1cm}
\includegraphics[height=4.2cm]{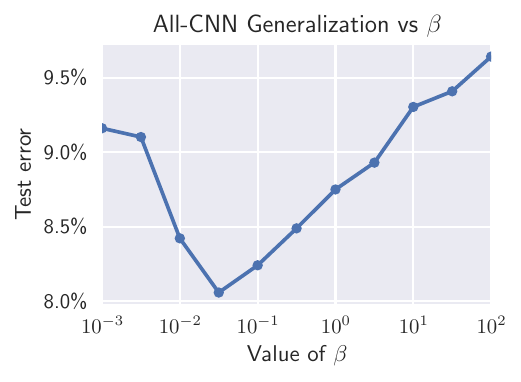}
\caption{
\label{fig:bias-variance}
{Plots of the test error obtained training the All-CNN architecture on CIFAR-10 (no data augmentation). \textbf{(Left)} Test error as we increase the number of weights in the network using weight decay but without any additional explicit regularization. Notice that increasing the number of weights the generalization error plateaus rather than increasing. \textbf{(Right) } Changing the value of $\beta$, which controls the amount of information in the weights, we obtain the characteristic curve of the bias-variance trade-off. This suggests that the quantity of information in the weights correlates well with generalization.}
}
\end{figure}

\subsection{Nuisance invariance}

Corollary \ref{cor:multi-layer-bound} shows that by decreasing the information in the weights $I(w;\D)$,
which can be done for example using \cref{eq:variational-regularizer},
the learned representation will be increasingly minimal,
and therefore insensitive to nuisance factors $n$,
as measured by $I(z;n)$.
Here, we adapt a technique from the GAN literature \cite{sonderby2016amortised} that allows us to explicitly measure
$I(z;n)$ and validate this effect, provided we can sample from the nuisance distribution $p(n)$ and from $p(x|n)$; that is, if given a nuisance $n$ we can generate data $x$ affected by that nuisance. Recall that by definition we have
\begin{align*}
I(z;n) &= \E_{n\sim p(n)} \KL{p(z|n)}{p(z)} \\
&= \E_{n\sim p(n)} \E_{z\sim p(z|n)} \log [p(z|n)/p(z)].
\end{align*}
To approximate the expectations via sampling
we need a way to approximate
the likelihood ratio $\log p(\z|\n)/p(\z)$.
This can be done as follows:
Let $D(z; n)$ be a binary discriminator that given the representation $z$ and the nuisance $n$ tries to decide whether $z$ is sampled from the posterior distribution $p(z|n)$ or from the prior $p(z)$. Since by hypothesis we can generate samples from both distributions, we can generate data to train this discriminator. Intuitively, if the discriminator is not able to classify, it means that $z$ is insensitive to changes of $n$. Precisely, since the optimal discriminator is
\[
D^*(z;n) = \frac{p(z)}{p(z) + p(z|n)},
\]
if we assume that $D$ is close to the optimal discriminator $D^*$, we have 
\[
 \log \frac{p(z|n)}{p(z)} = \log \frac{1-D^*(z;n)}{D^*(z;n)}
 \simeq \log \frac{1-D(z;n)}{D(z;n)}.
\]
therefore we can use $D$ to estimate the log-likelihood ratio, and so also the mutual information $I(z;n)$.
Notice however that this comes with no guarantees on the quality of the approximation.

To test this algorithm, we add random occlusion nuisances to MNIST digits (\Cref{fig:nuisances}). In this case, the nuisance $n$ is the occlusion pattern, while the observed data $x$ is the occluded digit.
For various values of $\beta$, we train a classifier on this data in order to learn a representation $z$,
and, for each representation obtained this way, we train a discriminator as described above and we compute the resulting approximation of $I(z;n)$.
The results in \Cref{fig:nuisances} show that decreasing the information in the weights makes the representation increasingly more insensitive to $n$.

\begin{figure}
\centering
\begin{subfigure}{1.2cm}
\centering
\includegraphics[height=.8cm]{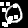} \\
\includegraphics[height=.8cm]{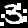} \\
\includegraphics[height=.8cm]{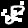}

\vspace{.6cm}
\end{subfigure}
\begin{subfigure}{6.5cm}
\centering
\includegraphics[height=4cm]{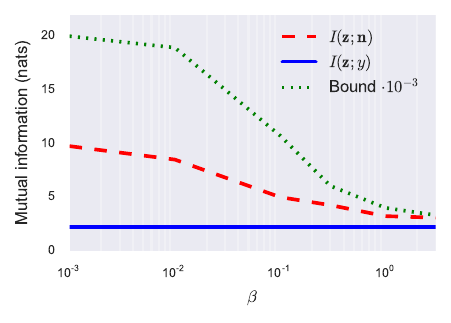}
\end{subfigure}

\vspace*{-.7em}

\caption{
\label{fig:nuisances}
\textbf{(Left)} A few training samples generated adding nuisance clutter $n$ to the MNIST dataset.
\textbf{(Right)} Reducing the information in the weights makes the representation $z$ learned by the digit classifier increasingly invariant to nuisances ($I(n;z)$ decreases), while sufficiency is retained ($I(z; y) = I(x; y)$ is constant). As expected, $I(z;n)$ is smaller but has a similar behavior to the theoretical bound 
in \Cref{cor:multi-layer-bound}.
}
\end{figure}

\section{Discussion and conclusion}

In this work, we have presented bounds, some of which are tight, that connect the amount of information in the weights, the amount of information in the activations, the invariance property of the network, and the geometry of the residual loss. These results leverage the structure of deep networks, in particular the multiplicative action of the weights, and the Markov property of the layers. This leads to the  surprising result 
that reducing  information stored in the weights about the past (dataset) results in desirable properties of the learned internal representation of the test datum (future).

Our notion of representation is intrinsically stochastic. This simplifies the computation as well as the derivation of information-based relations. However, note that even if we start with a deterministic representation $w$, Proposition \ref{prop:flat-minima} gives us a way of converting it to a stochastic representation whose quality depends on the flatness of the minimum. Our theory uses, but does not depend on, the Information Bottleneck Principle, which dates back to over two decades ago, and can be re-derived in a different frameworks, for instance PAC-Bayes, which yield the same results and additional bounds on the test error.

This work focuses on the inference and learning of optimal representations, that seek to get the most out of the data we have for a specific task. This does not guarantee a good outcome since, due to the Data Processing Inequality, the representation can be easier to use but ultimately no more informative than the data themselves. An orthogonal but equally interesting issue is how to get the most informative data possible, which is the subject of active learning, experiment design, and perceptual exploration. Our work does not address transfer learning, where a representation trained to be optimal for a task is instead used for a different task, which will be subject of future investigations.

\subsubsection*{Acknowledgments}
Supported by ONR N00014-17-1-2072, ARO W911NF-17-1-0304, AFOSR FA9550-15-1-0229 and FA8650-11-1-7156. We wish to thank our reviewers and David McAllester, Kevin Murphy, Alessandro Chiuso for the many insightful comments and suggestions.

\nocite{neklyudov2017structured}
\nocite{molchanov2017variational}
\nocite{lecun1998gradient}
\nocite{radford2015unsupervised}
\nocite{berberian1988}
\nocite{springenberg2014striving}
\nocite{clevert2015fast}

\bibliography{bibliography}

\clearpage

\appendix

\section{Details of the experiments}

\subsection{Random labels}

We use a similar experimental setup as \cite{zhang2016understanding}. In particular, we train a small version of AlexNet on a 28$\times$28 central crop of CIFAR-10 with completely random labels. The dataset is normalized using the global channel-wise mean and variance, but no additional data augmentation is performed. The exact structure of the network is in \Cref{table:networks}. As common in practice we use batch normalization before all the ReLU nonlinearities, except for the first layer. Optimization of the IB Lagrangian loss function is performed similarly to \cite{kingma2015variational} and \cite{molchanov2017variational}. We found that constraining the variance $\alpha_i$ of the weights to be the same for all weights in the same filter helps stabilizing the training process. We train with learning rates $\eta \in \{0.02,0.005\}$ and select the best performing network of the two. Generally, we found that a higher learning rate is needed to overfit when the number of training samples $N$ is small, while a lower learning rate is needed for larger $N$. We train with SGD with momentum $0.9$ for $360$ epochs reducing the learning rate by a factor of $10$ every $140$ epochs. We use a large batch-size of $500$ to minimize the noise coming from SGD. No weight decay or other regularization methods are used.

The final plot is obtained by triangulating the convex envelope of the data points, and by interpolating their value on the resulting simplexes. Outside of the convex envelope (where the accuracy is mostly constant), the value was obtained by inpainting.

To measure the information content of the weights as the percentage of corrupted labels varies, we fix $\beta=0.1$, $N=30000$ and $\eta=0.005$ and train the network on different corruption levels with the same settings as before.

To test the phase transition on multiple architectures, we train the Small AlexNet, the AllCNN network and a ResNet (see \Cref{table:networks}). For all architectures, we train with $N=10000$ random labels, $\eta=0.05$ and different values of $\beta$ log-uniformly spaced in $[10^{-2}, 10^2]$.

\subsection{Bias-variance trade-off}

For this experiment we train the AllCNN architecture (\Cref{table:networks}) on the CIFAR-10 dataset with ZCA whitening \cite{krizhevsky2009learning} and without any additional data augmentation. First, we train a standard network and change the number of filters (we multiplying the number of filters of all layers by the same constant) and train with $\eta=0.05$, batch size $128$, weight decay $0.001$. Then, we use the standard number of layers and train instead with the IBL loss function with different values of $\beta$.

\subsection{Nuisance invariance}

The cluttered MNIST dataset is generated by adding ten $4\times4$ squares uniformly at random on the digits of the MNIST dataset \citep{lecun1998gradient}.
For each level of $\beta$, we train the classifier in \Cref{table:networks} on this dataset. The weights of all layers, excluding the first and last one, are threated as a random variable with multiplicative Gaussian noise (\Cref{sec:gaussian-noise}) and optimized using the local reparameterization trick of \cite{kingma2015variational}. We use the last convolutional layer before classification as representation $z$.

The discriminator network used to estimate the log-likelihood ratio is constructed as follows: the inputs are the nuisance pattern $n$, which is a $28{\times}28{\times}1$ image containing 10 random occluding squares, and the 7$\times$7$\times$192 representation $z$ obtained from the classifier.
First we preprocess $n$ using the following network: \texttt{conv 48 $\to$ conv 48 $\to$ conv 96 s2 $\to$ conv 96 $\to$ conv 96 $\to$ conv 96 s2}, where each conv block
is a 3$\times$3 convolution followed by batch normalization and ReLU. Then, we concatenate the $7{\times}7{\times}96$ result with $z$ along the feature maps, and the final discriminator output is obtained by applying the following network: \texttt{conv 192 $\to$ conv 192 $\to$ conv 1$\times$1$\times$192 $\to$ conv 1{$\times$}1{$\times$}1 $\to$ AvgPooling 7$\times$7 $\to$ sigmoid}.

\begin{table}
\small
\centering
\vspace{0pt}
\begin{tabular}{|c|}
    \hline
    Input 32x32 \\
    \hline 
    conv 64 \\
    ReLU \\
    \hline
    MaxPool 2x2 \\
    \hline
    conv 64 + BN\\
    ReLU \\
    \hline
    MaxPool 2x2 \\
    \hline
    FC 3136x384 + BN \\
    ReLU \\
    \hline
    FC 384x192 + BN\\
    ReLU \\
    \hline
    FC 192x10 \\
    \hline
    softmax \\
    \hline
\end{tabular}
\hspace{.03\linewidth}
\begin{tabular}{|c|}
    \hline
    Input 28x28 \\
    \hline 
    conv 96 + BN + ReLU\\
    conv 96 + BN + ReLU\\
    \hline
    conv 192 s2 + BN + ReLU \\
    \hline
    conv 192 + BN + ReLU\\
    conv 192 + BN + ReLU\\
    \hline
    conv 192 s2 + BN + ReLU \\
    \hline
    conv 192 + BN + ReLU\\
    conv 192 + BN + ReLU\\
    \hline
    conv 1x1x10\\
    Average pooling 7x7 \\
    \hline
    softmax \\
    \hline
\end{tabular}
\hspace{.03\linewidth}
\begin{tabular}{|c|}
    \hline
    Input 28x28 \\
    \hline 
    conv 64\\
    \hline
    block 64 s1\\
    \hline
    block 128 s2\\
    \hline
    block 256 s3\\
    \hline
    block 512 s3\\
    \hline
    Average pooling 4x4 \\
    linear 10\\
    \hline
    softmax \\
    \hline
\end{tabular}
\caption{\label{table:networks}
\textbf{(Left)} The Small AlexNet model used in the random label experiment, adapted from \cite{zhang2016understanding}. All convolutions have a $5{\times}5$ kernel.
The use of batch normalization makes the training procedure more stable, but did not significantly change the results of the experiments.
\textbf{(Center)}
All Convolutional Network \citep{springenberg2014striving} used as a classifier in the experiments. All convolutions but the last one use a $3{\times}3$ kernel, ``s2'' denotes a convolution with stride 2. The final representation we use are the activations  of the last ``conv 192'' layer.
\textbf{(Right)} The ResNet architecture \citep{he2016deep} on which we test the phase transition. Each block with \texttt{f} filters and stride $s$ is structured as follows: \texttt{BN -> ReLU -> conv f stride s -> BN -> ReLU -> conv f} with a skip connection between first ReLU and the output.
}
\end{table}

\subsection{Visualizing the representation}

Even when we cannot generate data affected by nuisances like in the previous section, we can still visualize the information content of $z$ to 
learn what nuisances are discarded in the representation. To this end, given a representation $z$, we want to learn to sample from a distribution
$q(\hat{x}|z)$ of images that are maximally likely to have $z$ as their representation. Formally, this means that we want a distribution $q(\hat{x}|z)$ that maximizes the amortized maximum a posteriori estimate of $z$:
\begin{align*}
\E_z \E_{\hat{x} \sim q(\hat{x}|z)}[\log p(\hat{x}|z)] =&
\E_{z} \underbrace{\E_{\hat{x} \sim q(\hat{x}|z)} [ \log p(z|\hat{x})]}_{\text{Reconstruction error}} + \underbrace{\E_{\hat{x} \sim q(\hat{x})} [\log p(\hat{x})]}_{\text{Distance from prior}} + C.
\end{align*}
Unfortunately, the term $p(\hat{x})$ in the expression is difficult to estimate. However, \cite{sonderby2016amortised} notice that the modified gain function
\begin{multline*}
\E_z \E_{\hat{x} \sim q(\hat{x}|z)}[\log p(\hat{x}|z)] + H(p(\hat{x})) =
\E_{z} \E_{\hat{x} \sim q(\hat{x}|z)} [ \log p(z|\hat{x})] - 
\KL{q(\hat{x})}{p(\hat{x})} + C,
\end{multline*}
differs from the amortizes MAP only by a term $H(p(\hx))$, which has the positive effect of improving the exploration of the reconstruction, and contains the term $\KL{q(\hat{x})}{p(\hat{x})}$, which can be estimated  easily using the discriminator network of a GAN \cite{sonderby2016amortised}. To maximize this gain, we can simply train a GAN with an additional reconstruction loss $-\log p(z|\hat{x})$.

To test this algorithm, we train a representation $z$ to classify the 40 binary attributes in the CelebA face dataset \citep{yang2015facial}, and then use the above loss function to train a GAN network to reconstruct an input image $\hat{x}$ from the representation $z$. The results in \Cref{fig:faces} show that, as expected, increasing the value of $\beta$, and therefore reducing $I(w;\D)$, generates samples that have increasingly more random backgrounds and hair style (nuisances), while retaining facial features. In other words, the representation $z$ is increasingly insensitive to nuisances affecting the data, while information pertaining the task is retained in the reconstruction $\hat{x}$.

\begin{figure*}
\centering
\includegraphics[width=.44\linewidth]{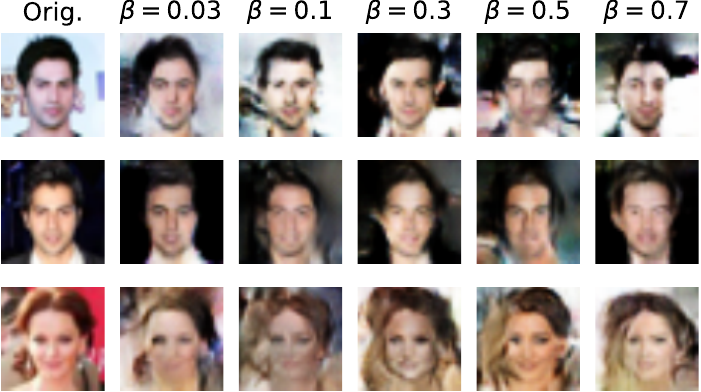}
\hspace{.03\linewidth}
\includegraphics[width=.44\linewidth]{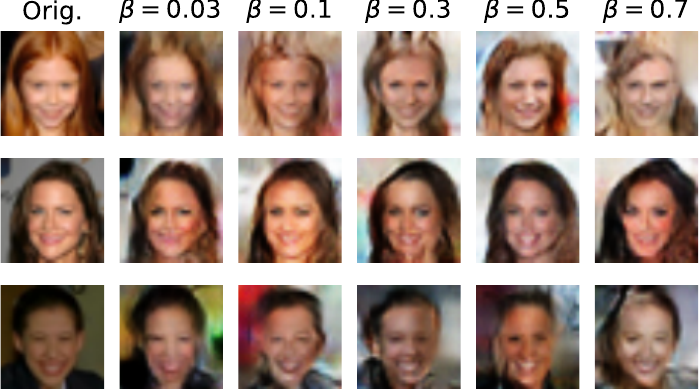}
\caption{
\label{fig:faces}
For different values of $\beta$, we show the image $\hat{x}$ reconstructed from a representation $z \sim  p(z|x)$ of the original image $x$ in the first column. For small $\beta$, $z$ contains more information regarding $x$, thus the reconstructed image $\hat{x}$ is close to $x$, background included. Increasing $\beta$ decreases the information in the weighs, thus the representation $z$ becomes more invariant to nuisances: Reconstructed image matches important details in $x$ that are preserved in $z$ (\emph{i.e.}, hair color, sex, expression), but background, hair style, and other nuisances are generated anew.
}
\end{figure*}

More precisely, we first train a classifier on the images from the CelebA datasets resized to 32$\times$32, where the task is to recover the 40 binary attributes associated to each image. The classifier network is the same as the one in \Cref{table:networks} with the following modifications: we use Exponential Linear Units \citep{clevert2015fast} for the activations, instead of ReLU, since invertible activations generally perform better when training a GAN, and we divide by two the number of output filters in all layers to reduce the training time. A sigmoid nonlinearity is applied to the final 40-way output of the network.

To generate the image $\hat{x}$ given the 8$\times$8$\times$96 representation $z$ computed by the classifier, we use a similar structure to DCGAN \citep{radford2015unsupervised}, namely \texttt{$z$ $\to$ conv 256 $\to$ ConvT 256s2 $\to$ ConvT 128s2 $\to$ conv 3 $\to$ tanh}, where \texttt{ConvT 256s2} denotes a transpose convolution with 256 feature maps  and stride 2. All convolutions have a batch normalization layer before the activations.
Finally, the discriminator network is given by \texttt{$\hat{x}$ $\to$ conv 64s2 $\to$ conv 128s2 $\to$ ConvT 256s2 $\to$ conv 1 $\to$ sigmoid}. Here, all convolutions use batch normalization followed by Leacky ReLU activations.

In this experiment, we use Gaussian multiplicative noise which is slightly more stable during training (\Cref{sec:gaussian-noise}). To stabilize the training of the GAN, we found useful to (1) scale down the ``reconstruction error'' term in the loss function and (2) slowly increase the weight of the reconstruction error up to the desired value during training.

\section{Gaussian multiplicative noise}
\label{sec:gaussian-noise}

In developing the theory, we chose to use log-normal multiplicative noise for the weights: The main benefit is that with this choice
the information in the weights $I(w;\D)$ can be expressed in closed form, up to an arbitrary constant $C$ which does not matter during the optimization process (but see also \cite{neklyudov2017structured} for a principled approach to this problem that uses a proper log-uniform prior).
Another possibility, suggested by \citet{kingma2015variational} is to use Gaussian multiplicative noise with mean 1.
Unfortunately, there is no analytical expression for $I(w;\D)$ when using Gaussian noise, but $I(w;\D)$ can still be approximated numerically  with high precision \citep{molchanov2017variational}, and it makes the training process slightly more stable. The theory holds with minimal changes also in this case, and we use this choice in some experiments.

\section{Proofs of theorems}
\label{sec:proofs}

\begin{lem}[Task-nuisance decomposition]
\label{lemma:task-nuisance-proof}
Given a joint distribution $p(x,y)$, where $y$ a discrete random variable, we can always find a random variable $n$ independent of $y$ such that $x=f(y,n)$, for some deterministic function $f$.
\end{lem}
\begin{proof}
Fix $n \sim \Unifdist(0,1)$ to be the uniform 
distribution on $[0,1]$.
We claim that, for a fixed value of $y$, there is a function $\Phi_y(n)$
such that $x|y = {\Phi_y}_* (n)$, where $({\cdot})_*$ denotes the push-forward map of measures.
Given the claim, let $\Phi(y,n) = (y,\Phi_y(n))$.
Since $y$ is a discrete random variable,
$\Phi(y,n)$ is easily seen to be a measurable function and by construction $(x,y) \sim \Phi_*(y,n)$. To see the claim, notice that, since there exists a measurable isomorphism between $\R^n$ and $\R$ (Theorem 3.1.1 of \citet{berberian1988}),
we can assume without loss of generality that $x \in \R$.
In this case, by definition, we can take $\Phi_y(n) = F_y^{-1}(n)$ where $F_y(t) = \P[x<t\,|\,y]$ is the cumulative distribution function of $p(x|y)$.
\end{proof}

\begin{prop}[Invariance and  minimality]
\label{prop:invariance-minimality-proof}
Let $n$ be a nuisance for the task $y$ and let $z$ be a sufficient representation of the input $x$. Suppose that 
$z$ depends on $n$ only through $x$ (\emph{i.e.,}  $n\to x\to z$). Then,
\[
I(z;n) \leq I(z;x) - I(x;y).
\]
Moreover, there exists a nuisance $n$ such that
equality holds up to a (generally small) residual $\epsilon$
\[
I(z;n) = I(z;x) - I(x;y) - \epsilon,
\]
where  $\epsilon := I(z;y|n) - I(x;y)$. In particular
$0 \leq \epsilon \leq H(y|x) $, and $\epsilon=0$ whenever $y$ is a deterministic function of $x$.
 Under these conditions, a sufficient statistic $z$ is invariant (maximally insensitive) to nuisances if and only if it is minimal.
\end{prop}
\begin{proof}
By hypothesis, we have the Markov chain $(y,n) \to x \to z$; therefore, by the DPI, we have $I(z;y,n) \leq I(z;x)$.
The first term can be rewritten using the chain rule 
as $I(z;y,n) = I(z;n) + I(z;y|n)$, giving us
\[
I(z;n) \leq I(z;x) - I(z;y|n).
\]
Now, since $y$ and $n$ are independent, $I(z;y|n) \geq I(z;y)$. In fact,
\begin{align*}
I(z;y|n) &= H(y|n) - H(y|z,n) \\
&= H(y) - H(y|z,n) \\
&\geq H(y) - H(y|z) = I(y;z).
\end{align*}
Substituting in the inequality above, and using the fact that $z$ is sufficient, we finally obtain
\[
I(z;n) \leq I(z;x) - I(z;y) = I(z;x) - I(x;y).
\]
Moreover, let $n$ be as in Lemma \ref{lemma:dist-factorization}.
Then, since $x$ is a deterministic function of $y$ and $n$, we have
\begin{align*}
I(z;x) &= I(z;n,y) = I(z;n) + I(z;y|n),
\end{align*}
and therefore
\[
I(z;n) = I(z;x) - I(z;y|n) = I(z;x) - I(x;y) - \epsilon.
\]
with $\epsilon$ defined as above. Using the sufficiency of $z$, the previous inequality for $I(z;y|n)$,
the DPI, we get the chain of inequalities
\begin{align*}
\epsilon &= I(z;y|n) - I(x;z) \leq I(x;y | n) - I(x;y) \\
&\leq H(y|n) - H(y|n,z) - H(y) + H(y|x)\\
&\leq H(y) - H(y | n,z) - H(y) + H(y|x)\\
&= H(y|x) - H(y | n,z) \\
&\leq H(y|x)
\end{align*}
from which we obtain the desired bounds for $\epsilon$.
\end{proof}

While the proof of the following theorem is quite simple, some clarifications on the notation are in order: We assume, following a Bayesian perspective, that the data is generated by some generative model $p(\x,\y|\theta)$, where the parameters $\theta$ of the model are sampled from some (unknown) prior $p(\theta)$. Given the parameters $\theta$, the training dataset  $\D = (\x,\y) \sim p(x,y|\theta)$ is composed of i.i.d. samples from the unknown distribution $p(x,y|\theta)$. The output of the training algorithm on the dataset $\D$ is a (generally simple, \textit{e.g.}, normal or log-normal) distribution $q(w|\x,\y)$ over the weights. Putting everything together, we have a well-defined joint distribution $p(\x,\y,\theta,w) = p(\theta)p(\x,\y|\theta)q(w|\x,\y)$.

Given  the weights $w$, the network then defines an inference distribution $q(\y|\x,w)$, which we know and can compute explicitly. Another distribution, which instead we do not know, is $p(\y|\x,w)$, which is obtained from $p(\x,\y,\theta,w)$ and express the optimal inference we could perform on the labels $\y$ using the information contained in the weights. In a well trained network, we want the distribution approximated by the network to match the optimal distribution $q(\y|\x,w) = p(\y|\x,w)$.

Finally, recall that the conditional entropy is defined as
\[
H_p(y|z) := \E_{y,z \sim p(y,z)} [-\log p(y|z)],
\]
where $z$ can be one random variable or a tuple of random variables. When not specified, it is assumed that the cross-entropy is computed with respect to unknown underlying data distribution $p(\x,\y,w,\theta)$.
Similarly, the conditional cross-entropy is defined as
\begin{align*}
H_{p,q}(y|z) :=& \E_{y,z \sim p(y,z)}[-\log q(y|z)]\\
=& \E_{y,z \sim p(y,z)}[-\log p(y|z)] + \E_{y,z \sim p(y,z)}[\log \frac{p(y|z)}{q(y|z)}] \\
=& H_p(y|z) + \E_{z \sim p(z)}{\KL{p(y|z)}{q(y|z)}}.
\end{align*}

\begin{prop}[Information Decomposition]
\label{prop:information-decomposition-proof}
Let $\D=(\x,\y)$ denote the training dataset, then for any training procedure, we have 
\begin{equation}
H_{p,q}(\y|\x,w)
= H(\y|\x,\theta) + I(\theta;\y|\x, w) +
\E_{\x,w}\KL{p(\y|\x,w)}{q(\y|\x,w)} - I(\y;w|\x, \theta).
\label{eq:decomposition}
\end{equation}
\end{prop}
\begin{proof}
Recall that cross-entropy can be written as
\[
H_{p,q}(\y|\x, w) = H_p(\y|\x,w) + \E_{\x,w}\KL{p(\y|\x,w)}{q(\y|\x,w)},
\]
so we only have to prove that
\[
H_p(\y|\x,w) = H_p(\y|\x,\theta) + I(\y;\theta|\x, w) - I(\y;w|\x, \theta),
\]
which is easily done using the following identities:
\begin{align*}
I(\y;\theta|\x,w) &= H_p(\theta,\y|w) - H_p(\y|\theta,\x,w),\\
I(\y;w|\x,\theta) &= H_p(\y|\x,\theta) - H_p(\y|\x,\theta,w).
\end{align*}
\end{proof}

\begin{prop}[Information in the weights]
\label{prop:information-weight-proof}
Under the previous modeling assumptions, the upper-bound to the information that the weights contain about the dataset is 
\[
I(w;\D) \leq \tilde{I}(w;\D) = - \half \sum_{i=1}^{\dim(w)} \log \alpha_i + C,
\]
where the constant $C$ is arbitrary due to the improper prior.
\end{prop}
\begin{proof}
Recall that we defined the upperbound $\tilde{I}(\w;\D)$ as 
\[
\tilde{I}(\w;\D) = \KL{q(w|\D)}{\tilde{q}(w)},
\]
where $\tilde{q}(w)$ is a factorized log-uniform prior. Since the KL divergence is reparametrization invariant, we have:
\begin{align*}
\KL{q(w|\D)}{\tilde{q}(w)} &= \KL{\log \N(\mu,\alpha)}{\log \Unifdist} \\
&= \KL{\N(\mu,\alpha)}{\Unifdist} \\
&= H(\N(\mu,\alpha)) + \text{const} \\
&= -\sum_{i=1}^{\dim(w)} \frac{1}{2} \log(\alpha_i) + \text{const},
\end{align*}
where we have used the formula for the entropy of a Gaussian and the fact that the KL divergence of a distribution from the uniform prior is the entropy of the distribution modulo an arbitrary constant.
\end{proof}

\begin{prop}[Flat minima have low information]
\label{prop:flat-minima-proof}
Let $\hat w$ be a local minimum of the cross-entropy loss $H_{p,q}(\y|\x,w)$, and let $\H$ be the Hessian at that point.
Then, for the optimal choice of the
posterior $w|\D = \epsilon \odot \hat w$ centered at $\hat w$ that optimizes the IB Lagrangian, we have
\[
I(w;\D) \leq \tilde{I}(w;\D) \leq \half K [\log \norm{w}_2^2 + \log \norm{H}_* - K\log (K^2 \beta/2)]
\]
where $K = \dim(w)$ and $\| \cdot \|_*$ denotes the nuclear norm.
\end{prop}
\begin{proof}
First, we switch to a logarithmic parametrization of the weights, and let $h := \log |w|$ (we can ignore the sign
of the weights since it is locally constant). In this parametrization, we can approximate the IB Lagrangian to second order as
\begin{align*}
\L =& \E_{h \sim p(h|\D)}[H_0 + [(h - h_0) \odot w]^T \H [(h - h_0) \odot w]  - \frac{\beta}{2} \sum_i  \log \alpha_i
\end{align*}
where $H_0 = H(\y | \x, \hat w)$. Now, notice that since $q(w|\D)$ is a log-normal distribution, we have $q(h|\D) \sim \N(h_0,\alpha)$.%
\footnote{Note that for simplicity we have ignored the offset $\alpha/2$ in the mean of the log-normal distribution.}
Therefore, can compute the expectation exactly as
\[
\L = H_0 + \sum_{i=1}^{\dim(w)} \alpha_{i} w^2_i \H_{ii} \\
 - \frac{\beta}{2} \sum_i \log \alpha_i.
\]
Optimizing w.r.t. $\alpha_i$ we get
\[
\alpha_i = \frac{\beta}{2 w_i^2 \H_{ii}},
\]
and plugging it back in the expression for $\tilde{I}(w;\D)$ that we obtained in the previous proposition, we have
\[
\tilde{I}(w;\D) = -\half \sum_i \log \alpha_i = \half \sum_i \log(w_i^2) + \log(\H_{ii}) - \log (\beta/2).
\]
Finally, by Jensen's inequality, we have
\begin{align*}
\tilde{I}(w;\D) &\leq \half K [\log(\sum_i w_i^2) + \log(\sum_i \H_{ii}) - \log (K^2 \beta/2)]\\
&= \half K [\log(\norm{w}_2^2) + \log(\norm{H}_*) - \log (K^2 \beta/2)],
\end{align*}
as we wanted.
\end{proof}

\begin{prop}
\label{prop:information-exact-computation}
Let $z=Wx$, and assume as  before $W = \epsilon \odot \hat{W}$, with 
$\epsilon_{i,j} \sim  \log \N(-\alpha_i/2, \alpha_i)$.
Further assume that the marginals of $q(z)$ and $q(z|x)$ are both approximately Gaussian
(which is reasonable for large $\dim(x)$ by the Central Limit Theorem). Then,
\[
I(z;x) + \TC(z)
=  - \half \sum_{i=1}^{\dim(z)} \E_x  \log \frac{\tilde{\alpha}_i \hat{W}_i^2 \cdot  x^2}{\hat{W}_i \cdot \Cov(x) \hat{W}_i  + \tilde{\alpha}_i \hat{W}_i^2 \cdot \E (x^2)},
\]
where $W_i$ denotes the $i$-th row of the matrix $W$, and $\tilde{\alpha_i}$ is the noise variance $\tilde{\alpha}_i = \exp(\alpha_i) - 1$. In particular, $I(z;x) + \TC(z)$ is a monotone decreasing function of the weight variances $\alpha_i$.
\end{prop}
\begin{proof}
First, we consider the case in which $\dim(z)=1$, and so $w := W$ is a single row vector.
By hypothesis, $q(z)$ is approximately Gaussian, with mean and variance
\begin{align*}
\mu_1 &:= \E[z] = \E[\sum_i \epsilon_i \hat{w}_i x_i] = \sum_i  \hat{w}_i \E[x_i] 
= \hat{w} \cdot \E[x] \\
\sigma_1^2 &:= \var[z] = \E[(\sum_i \epsilon_i \hat{w}_i x_i)^2] - (\E[\sum_i \epsilon_i \hat{w}_i x_i])^2, \\
&= \E[\sum_{i,j} \epsilon_i \epsilon_j  \hat{w}_i \hat{w}_j  x_i x_j] - \sum_{i,j}  \hat{w}_i\hat{w}_j \E[x_i] \E[x_j] \\
&= \tilde{\alpha} \sum_i \hat{w}_i^2 \E[x_i]^2 + \sum_{i,j} \hat{w}_i\hat{w}_j \pa{\E[x_i x_j] - \E[x_i] \E[x_j]} \\
&=  \tilde{\alpha} \hat{w}^2 \cdot \E[x^2] + \hat{w} \cdot \Cov(x) \hat{w}.
\end{align*}
A similar computation gives us mean and variance of $q(z|x)$:
\begin{align*}
\mu_0 := \E[z|x] &= \hat{w} \cdot x, \\
\sigma_0^2 := \var[z|x] &= \tilde{\alpha} \hat{w}^2 \cdot x^2.
\end{align*}
Since we are assuming $\dim(z)=1$, we trivially have $\TC(z)=0$, so we are only left with $I(z;x)$ which is given by
\begin{align*}
I(z;x) &= \E_x \KL{q(z|x)}{q(z)} \\
&= \E_x \KL{\N(\mu_0, \sigma_0^2)}{\N(\mu_1, \sigma_1^2)} \\
&= \half \E_x \frac{\tilde{\alpha} \hat{w}^2 \cdot x^2  + (\hat{w} \cdot x - \hat{w} \cdot \E[x] )^2}{\sigma_1^2} - 1 - \log \frac{\sigma_0^2}{\sigma_1^2} \\
&= - \half \E_x \log \frac{\tilde{\alpha} \hat{w}^2 \cdot x^2}{\hat{w} \cdot \Cov(x) \hat{w} + \tilde{\alpha} \hat{w}^2 \cdot \E[x^2]}.
\end{align*}
Now, for the general case of $\dim(z)\geq 1$, notice that 
\begin{align*}
I(\z;\x) + \TC(\z) &= \E_x \KL{\prod_k q(z_i|\x)}{\prod_k q(z_i)}\\
&= \sum_{i=1}^{\dim(z)} \E_x \KL{q(z_i|\x)}{q(z_i)},
\end{align*}
where $q(z_i)$ is the marginal of the $k$-th component of $z$.
We can then use the previous result for each component separately,
and sum everything to get the desired identity.
\end{proof}

\begin{prop}[Uniform bound for one layer]
\label{prop:uniform-bound-proof}
Let $z=Wx$, where $W = \epsilon \odot \hat{W}$, where $\epsilon_{i,j} \sim \log \N(-\alpha/2,\alpha)$; assume that the components of $x$ are uncorrelated, and that their kurtosis is uniformly bounded.
Then, there is a strictly increasing function $g(\alpha)$ s.t. we have the uniform bound
\[
g(\alpha) \leq \frac{I(x;z) + TC(z)}{\dim(z)} \leq g(\alpha) + c,
\]
where $c=O(1/\dim(x))\leq 1$,
$g(\alpha)= \log\,(1-e^{-\alpha})/2$
and $\alpha$ is related to $I(w;\D)$ by  $\alpha = \exp\set{-I(W;\D)/\dim(W)}$.
In particular, $I(x;z) + TC(z)$ is tightly bounded by $I(W;\D)$ and increases strictly with it.
\end{prop}
\begin{proof}
To simplify the notation we do the case $\dim z=1$, the general case being identical.
Let $w:=W$ be the only row of $W$. First notice that, since $x$ is uncorrelated, we have
\[
\hat{w} \cdot \Cov(x) \hat{w} = \sum_i w_i^2 (\E[x_i^2] - \E[x_i]^2) \leq w^2 \cdot \E[x^2]
\]
Therefore, 
\begin{align*}
I(x;z) =& - \half \E_x \log \frac{\tilde{\alpha} \hat{w}^2 \cdot x^2}{\hat{w} \cdot \Cov(x) \hat{w} + \tilde{\alpha} \hat{w}^2 \cdot \E[x^2]} \\
\leq& - \half \E_x \log \frac{\tilde{\alpha} \hat{w}^2 \cdot x^2}{(1 + \tilde{\alpha}) \hat{w}^2 \cdot \E[x^2]} \\
=& \half \log (1+\tilde{\alpha}^{-1})  \\ &-\half  \E_x \log \bra{1 + \frac{\hat{w}^2 \cdot (x^2 - \E[x^2] )}{\hat{w}^2 \cdot \E[x^2]}}.
\end{align*}
To conclude, we want to approximate the expectation of the logarithm using a Taylor expansion, but
we first need to check that the variance of the term inside the logarithm is low,
which is where we need the bound on the kurtosis.
In fact, since the kurtosis is bounded, there is some constant $C$ such that for all $i$
\[
\frac{\E (x_i^2 - \E[x_i^2])^2}{\E[x_i^2]^2} \leq C.
\]
Now,
\begin{align*}
\var \frac{\hat{w}^2 \cdot (x^2 - \E[x^2])}{\hat{w}^2 \cdot \E[x^2]} &= \frac{\sum_i \hat{w}_i^4 \E (x^2 - \E[x^2])^2}{\sum_{i,j} \hat{w}_i^2 \hat{w}_j^2 \E[x_i^2]\E[x_j^2]} \\
&\leq C \frac{\sum_i \hat{w}_i^4 \E[x_i^2]^2}{\sum_{i,j} \hat{w}_i^2 \hat{w}_j^2 \E[x_i^2]\E[x_j^2]} \\
&= O(1/\dim(x)).
\end{align*}
Therefore, we can conclude
\[
I(x;z) \leq \half \log (1+\tilde{\alpha}^{-1}) + O(1/\dim(x)).
\]

\end{proof}

\begin{cor}[Multi-layer case]
\label{cor:multi-layer-proof}
Let $W^k$ for $k=1,...,L$ be weight matrices, with $W^k = \epsilon^k \odot \hat{W}^k$ and $\epsilon^k_{i,j}=\log \N(-\alpha^k/2,\alpha^k)$, and let $z_{i+1} = \phi(W^k z_k)$,
where $z_0 = x$ and $\phi$ is any nonlinearity. Then,
\[
I(z_L;x) \leq \min_{k<L} \set{\dim(z_k) \big[g(\alpha^k) + 1 \big] }
\]
where $\alpha^k = \exp\set{-I(W^k;\D)/\dim(W^k)}$.
\end{cor}
\begin{proof}
Since we have the Markov chain $x \to z_1 \to \ldots \to z_L$,
by the Data Processing Inequality we have
$I(z_L;x) \leq \min \set{I(z_L;z_{L-1}), I(z_{L-1};x)}$.
Iterating this inequality, we have
\[
I(z_L;x) \leq \min_{k<L} I(z_{k+1},z_k).
\]
Now, notice that $I(z_{k+1}; z_k) \leq I(\phi(W^k z_k);z_k) \leq I(W^k z_k;z_k)$,
since applying a deterministic function can only decrease the information.
But $I(W^k z_k;z_k)$ is exactly the quantity we bounded in Corollary \ref{cor:single-layer-bound},
leading us to the desired inequality.
\end{proof}

\section{Q\&A}

\paragraph {It is well-known that overfitting relates to the ``effective number of degrees of freedom'' that can be measured in a number of ways \citep[Chapter 7]{friedman2001elements}. Why should we use the information in the weights?} The information in the weights is indeed one particular choice of measure of complexity. One nice aspect is that it plays a central role in many different frameworks (minimum description length, variational inference, PAC-Bayes), and correlates well with the performance of a real network.

\paragraph{How do you compute the nuclear norm of the Hessian? It sounds expensive!} We do not need to. What we show is that, if the optimization algorithm happens to find flat minima (nuclear norm being a proxy), then it automatically limits the information in the weights - which promotes good generalization. However, if one wanted to approximate the trace of the Hessian, it could be done in linear time \citep[Prop 4.1]{bai1996some}.

\paragraph{A nuisance $n$ should convey no information on $y$ {\em given} $x$, so why imposing $I(y;n) = 0$ rather than $I(y;n|x)=0$?}
While $I(y;n|x) = 0$ may seem intuitively the right condition, it is actually too weak. Suppose for example that $y$ is a deterministic function of $x$ (\textit{i.e.}, the labels are perfectly determined by the data, as often is the case). Then, we would  have $I(y;n | x) = 0$ for any $n$, which would imply that everything is a nuisance for the task, which of course is not intuitively the case.

\paragraph{Why is the dataset a random variable, if we only have one realization of it?} The theory is almost identical in both the case of a fixed dataset and a randomly sampled dataset. We consider the case of a randomly sampled dataset  since it is simpler and at the same time more general.  Some expressions simplify slightly and are easier to interpret, but in the end a fixed training set is given either way.

\paragraph{The use of $KL(q(w|D)||\tilde{q}(w))$ where $q(w|D)$ is a ``posterior'' defined by a learning algorithm that returns $w$ and $\tilde{q}(w)$ is a log-uniform prior is the basic PAC-Bayes bound, that, however, gives an vacuous generalization bound due to the improper prior. A generalization bound could be stated in terms of the length of a finite interval approximation of log-uniform prior.} Indeed, this is the case. However, in the limit of the interval length going to infinity, the KL divergence would still be infinite, and the optimization would be slightly more complex. As simpler option to have a generalization-bound would be to use Gaussian prior and posterior \cite{dziugaite2017computing}. However, computing a good PAC-Bayes upper-bound is outside the scope of the paper, and the use of a non-informative, scale invariant prior matches the empirical behavior of networks and simplifies the theoretical analysis.

\paragraph{Of course a minimal representation should be invariant to nuisance variation since that's what it means to be minimal for the task.} While the result may be intuitive to some, compression and invariance are not the same thing, and we are unaware of an existing proof of the claim other than for special tasks like clustering, and for small perturbations. 

{\paragraph{ How do you compute the information in the weights, since you only have a sample (one set of weights that the network converged to for a given dataset)? And how do you optimize it? That looks hard!} We do not need to compute the information in the weights, since we can control it. Even if we do not do so explicitly, optimizing cross-entropy with SGD  yields the right solution (the solution that minimizes the IBL for the weights). We only have one sample of the weights if we anneal the learning rate to zero, but otherwise SGD produces a posterior distribution of the weights, even if we do not impose additional stochasticity. In many cases we do, for instance using Information Dropout (\cite{achille2016information}) or its simpler version, Dropout. A special case of the theory can be re-derived assuming an instantiated dataset and a set of weights, with the same results.
}

\noindent{\bf The use of $I(w;\D)$ does not seem to upper bound the PAC-Bayes bound. The direction of inequality is the opposite.}
The IB Lagrangian for the weights is a particular case that gives the sharpest PAC-Bayes bound. Choosing $q(w)$ requires some attention: Once we fix a training procedure, we can (in theory) compute/approximate the marginal distribution $q(w)$ over the stochasticity of the data (if any) as well as the stochasticity of the training procedure. This marginal can then be used in the PAC-Bayes bound: This gives the sharpest bound \citep{mcallester2013pac} and is also equivalent to an IBL since the term $\E_\D \KL{q(w|\D)}{q(w)}$ measures the mutual information between the training procedure and the weights. Since in general it is not possible to explicitly compute this bound, in practice we use less tight a bound based on a factorized approximation of the marginal. This opportunistic choice later turns out to play a role in disentanglement.

\end{document}